\documentclass[journal]{IEEEtran}

\usepackage[nocompress]{cite}  %%%%%%%

\usepackage{amsmath,amssymb,amsfonts}
\usepackage{algorithmic}
\usepackage{graphicx}
\usepackage{textcomp}
\usepackage{booktabs}  %
\usepackage{array}     %
\usepackage{algorithm}
\usepackage{amsthm}
\theoremstyle{plain}
\newtheorem{theorem}{Theorem}[section]
\newtheorem{proposition}[theorem]{Proposition}
\newtheorem{lemma}[theorem]{Lemma}
\newtheorem{corollary}[theorem]{Corollary}
\theoremstyle{definition}

\theoremstyle{remark}
\newtheorem{remark}[theorem]{Remark}
\usepackage{hyperref}
\usepackage[capitalize,noabbrev]{cleveref}
\usepackage{subcaption}
\usepackage{float}

\begin{document}

\title{ Exploring the Impact of Parameter Update Magnitude on Forgetting and
Generalization of Continual Learning }

%%%%
\author{

JinLi He, Liang Bai, Xian Yang

\IEEEcompsocitemizethanks{ 
\IEEEcompsocthanksitem J. He and L. Bai are with Institute of Intelligent Information Processing, Shanxi University, Taiyuan, 030006, China (Corresponding author: Liang Bai)\protect\\
Email: hejinli@sxu.edu.cn, bailiang@sxu.edu.cn
\IEEEcompsocthanksitem X. Yang is with Alliance Manchester Business School, The University of Manchester, Manchester, M13 9PL, UK. Email: xian.yang@manchester.ac.uk   }

}

\maketitle

%%%%%%%%%%%%%%%%%%%%%%%%%%%%%%%%%%%%%%%%%%%%%%%%%%%%%%%%%
\begin{abstract}

    The magnitude of parameter updates are considered a key factor in continual learning. However, most existing studies focus on designing diverse update strategies, while a theoretical understanding of the underlying mechanisms remains limited. Therefore, we characterize model's forgetting from the perspective of parameter update magnitude and formalize it as knowledge degradation induced by task-specific drift in the parameter space, which has not been fully captured in previous studies due to their assumption of a unified parameter space. By deriving the optimal parameter update magnitude that minimizes forgetting, we unify two representative update paradigms, frozen training and initialized training, within an optimization framework for constrained parameter updates. Our theoretical results further reveals that sequence tasks with small parameter distances exhibit better generalization and less forgetting under frozen training rather than initialized training. These theoretical insights inspire a novel hybrid parameter update strategy that adaptively adjusts update magnitude based on gradient directions. Experiments on deep neural networks demonstrate that this hybrid approach outperforms standard training strategies, providing new theoretical perspectives and practical inspiration for designing efficient and scalable continual learning algorithms.

\end{abstract}

%%%%%%%%%%%%%%%%%%%%%%%%%%%%%%%%%%%%%%%%%%%%%%%%%%%%%%%%%
\begin{IEEEkeywords}
Continual Learning, Learning Theory, Incremental Learning.
\end{IEEEkeywords}

%%%%%%%%%%%%%%%%%%%%%%%%%%%%%%%%%%%%%%%%%%%%%%%%%%%%%%%%%
\section{Introduction}

    \IEEEPARstart{T}{raditional} machine learning typically assume closed-environment settings in which critical factors in the learning process remain constant \cite{Machine1_2015,Machine2_2021,Machine3_2023}. However, data often arrive continuously in real-world applications, and collecting and storing all observed data becomes impractical due to escalating storage requirements, increasing computational costs, and potential privacy concerns \cite{OpenWorld1_2022,OpenWorld2_2023,OpenWorld3_2024}. A more practical alternative is to incrementally update the model in real time as new data become available \cite{incremental1_2020,incremental2_2022,incremental3_2023,incremental4_2023}.

    Unlike traditional approaches that capture static data distributions, continual learning \cite{continual1_2016,continual2_2024,continual3_2024_suvey} aims to learn from data that arrive sequentially with the goal of maintaining reliable performance across all tasks. This paradigm aims to mimic human learning by incrementally acquiring new skills throughout their lifecycle \cite{human1_2021,human2_2022,human3_2023}. However, a key challenge is catastrophic forgetting \cite{Catastrophic1_1989,Catastrophic2_1995,Catastrophic3_2017}, where models suffer a significant degradation in performance on previously learned data after adapting to new ones.

    In recent years, numerous methods have been proposed to address catastrophic forgetting \cite{survey1_2018,survey3_2020,survey4_2023}. One approach \cite{Regularization1_2016,Regularization5_2018,Regularization9_2024} constrains parameters to prevent deviation from the solution space of previous tasks. Another approach \cite{Architecture1_2016,Architecture2_2018,Architecture5_2023} dynamically allocates specific network parameters or modules to each incoming task. Additionally, incorporating previous data \cite{replay1_2016,replay4_2018,replay11_2024} by storing and replaying portions of previous datasets is a simple yet effective strategy. Although numerous methods have contributed to mitigate catastrophic forgetting, some studies suggest that existing algorithms still suffer from computational, memory, and storage inefficiencies~\cite{HowEfficient1_2023,Efficiency1_2023,Efficiency2_2024}. Some approaches even require more parameter updates than retraining from scratch~\cite{HowEfficient1_2023}. Therefore, it is essential to explore effective optimization paths on the parameter update magnitudes for practical applications.

    In this paper, we address fundamental questions driving this research domain: Is catastrophic forgetting solely caused by changes in a subset of parameters? When does freezing most parameters outperform training all parameters? While stability in prior tasks can be maintained through cross-task parameter sharing and additional task-specific parameters enhance plasticity for learning new tasks, does this also exacerbate forgetting? How much parameter updating is sufficient to achieve strong performance on new tasks?

    To address the above questions, a mathematical analysis is required to examine parameter-space separation and forgetting dynamics. However, characterizing parameter separation in sequential tasks is inherently more complex than in static scenarios, and variations in training strategies and optimization objectives further compound this difficulty. The main contributions of this study are summarized as follows:

    \begin{itemize}
      \item We characterize forgetting under frozen training and initialized training and derive the optimal update magnitude for minimizing forgetting. The results demonstrate that frozen training outperform initialized training when sequential tasks are close in parameter space.
      \item We propose an adaptive hybrid update framework that dynamically adjusts update magnitude according to the gradient direction in parameter space, enabling a finer and flexible trade-off between stability and plasticity.
      \item Experiments on deep neural networks with real-world datasets indicate that the proposed hybrid update framework consistently outperforms standard training strategies, demonstrating that our theoretical insights guide practical algorithm design in continual learning.
    \end{itemize}

%%%%%%%%%%%%%%%%%%%%%%%%%%%%%%%%%%%%%%%%%%%%%%%%%%%%%%%%%
\section{Previous Research}

    Empirical Research on Continual Learning. Existing continual learning methods can generally be categorized into regularization-based, replay-based, and architecture-based approaches \cite{survey2_2019,continual4_2024_survey,survey5_2023}. Regularization-based methods \cite{Regularization2_2017,Regularization4_2018,Regularization6_2019,Regularization8_2023} preserve previous knowledge by penalizing changes in parameters or prediction functions, and they typically require storing frozen copies of previous models as references. Weight regularization selectively constrains updates to parameters that are important for earlier tasks, while function regularization restricts intermediate representations or final predictions to mitigate forgetting through knowledge distillation. Replay-based methods \cite{replay2_2017,replay9_2020,replay12_2024,replay14_2025} revisit previous data distributions by storing or generating samples from previous tasks, and can be further divided into sample replay and generative replay. Sample replay stores a limited number of samples from previous tasks in a replay buffer and trains the model jointly with data from the current task. Generative replay, in contrast, introduces an auxiliary generative model to synthesize pseudo-replay data, with the generator itself being incrementally updated over time. Architecture-based methods \cite{Architecture3_2020,Architecture4_2022,Architecture6_2024,Architecture9_2025} allocate isolated parameters or submodules to each task during training. In recent years, continual learning built upon pre-trained models \cite{pretain1_2021,pretain2_2023,pretain4_2024} have enabled training across sequence tasks from a stronger initialization, further mitigating model's forgetting by leveraging transferable and generalizable features across tasks.

    Given that catastrophic forgetting primarily stems from parameter updates displacing previously learned knowledge, we analyze the following scenarios. When the model is fully frozen, it preserves performance on previously seen data but at the cost of being unable to acquire new knowledge; conversely, fine-tuning the model with new data allows it to capture current pattern and rapidly adapt to new tasks but often degrades prior experience. These contrasts motivate us to rethink: how can we prevent the loss of previously learned knowledge while training on new tasks, thereby enabling efficient and sustainable intelligent systems? And how many parameter updates are sufficient to learn a specific task while minimizing forgetting and achieving strong performance on incoming tasks?

    Theoretical Research on Continual Learning. Although empirical studies on continual learning have made substantial progress, related theoretical research \cite{theory1_2017,theory2_2022,theory3_2023} remains limited. In this context, \cite{theoryA_functionalregularisation} extends function regularization beyond weight regularization by modeling Gaussian processes as function priors for learning subsequent tasks. \cite{theoryB_teacherstudent} studies how teacher–student relationships influence knowledge distillation. \cite{theoryC_Howcatastrophic} connects continual learning with alternating projection and Kaczmarz methods, providing worst-case forgetting analyses. \cite{theoryD_theoreticalstudyOOD} decomposes class-incremental learning into intra-task prediction and task-identity prediction, linking the latter to out-of-distribution detection. \cite{theoryE_promptbased} formulates continual learning with pre-trained models as intra-task prediction, task-identity inference, and task-adaptive prediction. \cite{theoryF_Learnability} proves the learnability of class-incremental learning from a predictive probability viewpoint. \cite{theoryG_forgettingandgeneralization} analyzes forgetting and generalization errors in overparameterized models. \cite{qianyi} analyzed how the generalization performance of transfer learning evolves as the number of features increases, under both underfitting and overfitting conditions. They found that in scenarios with high noise levels and small true parameters, allocating additional redundant features to specific tasks leads to better generalization performance. \cite{theoryH_plasticity} identifies plasticity loss in long-sequence learning and proposes continual backpropagation. \cite{theoryI_Measuringforgetting} proposes new evaluation metric to quantify representation forgetting beyond prediction function forgetting.

    \setlength{\textfloatsep}{14pt}   %
    \begin{figure}[t]
        \centerline{\includegraphics[width=\columnwidth]{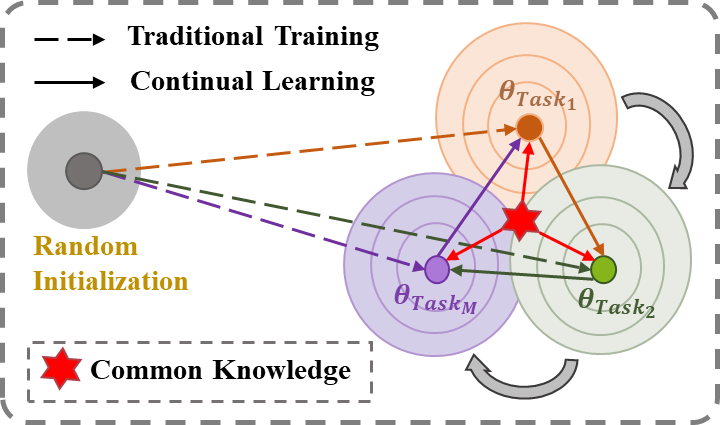}}
        \caption{
         Switching tasks shifts target parameter space, causing forgetting in continual learning. The red symbols shows that a shared space enables fast adaptation and knowledge retention.
      }
      \label{图片知识共享}
    \end{figure}

    However, existing theoretical research has not provided an explicit analysis of parameter space separation under different update strategies. This work differs in that it unifies two representative update paradigms, frozen training and initialized training, within an optimization framework that constrains parameter updates, and derives the optimal parameter update magnitude that minimizes forgetting, offering a perspective orthogonal to existing theoretical analyses.

%%%%%%%%%%%%%%%%%%%%%%%%%%%%%%%%%%%%%%%%%%%%%%%%%%%%%%%%%
\section{Problem Setup}

    In a continual learning scenario with $M$ sequential tasks, the dataset for task $k$ is denoted by $\mathcal{D}_k = \{(R_{(k)}^i, y_{(k)}^i)\}_{i=1}^{N_k}$, consisting of $N_k$ input-label pairs, where $k$ indexes the task identity. Starting from an initial parameter vector $\theta_{(0)}$, the model is trained sequentially, yielding parameter estimates $\theta_{(1)}, \theta_{(2)}, \ldots, \theta_{(M)}$. With analytical tractability, the Gaussian features and noise are adapted. The objective is to learn an optimal parameter update model under fixed model capacity that consistently satisfies: (1) a finite parameter overhead as the number of tasks increases, and (2) optimal magnitude of parameter updates across training stages. 
    
    To achieve the above objectives, the parameter space is split into shared and task-specific parts, with task-specific parts quantified to minimize catastrophic forgetting. The optimized model is denoted as $y_{(k)} = X_{(k)}^\top u_{(k)} + Z_{(k)}^\top q_{(k)} + \varepsilon_{(k)}$, where $X_{(k)}\in \mathbb R^{p\times n_{(k)}}$ represents the features associated with the shared parameters $u_{(k)} \in \mathbb R^{p}$, and $Z_{(k)} \in \mathbb R^{p_{(k)} \times n_{(k)}}$ denotes features associated with the task-specific parameters $q_{(k)} \in \mathbb R^{p_{(k)}}$. Performance on task $k$ is measured by squared loss, following the metric from \cite{theory4_2024,theoryG_forgettingandgeneralization,survey4_2023}.

    Note: For clarity, we describe the notation commonly used in this paper. Matrices are denoted by uppercase letters $\mathbf{X}$, and vectors by lowercase letters $\mathbf{x}$. For vectors $\mathbf{u}$ and $\mathbf{q}$, their inner product is $\langle \mathbf{u}, \mathbf{q} \rangle$. For matrices $\mathbf{X}$ and $\mathbf{Z}$, their inner product is denoted by $\langle \mathbf{X}, \mathbf{Z} \rangle_F$. Let $\|\mathbf{a}\|_2$ denote the $\ell_2$-norm of vector $\mathbf{a}$, and let $\|\mathbf{V}\|$ denote the spectral norm of matrix $\mathbf{V}$.

%%%%%%%%%%%%%%%%%%%%%%%%%%%%%%%%%%%%%%%%%%%%%%%%%%%%%%%%%
\section{Frozen Training vs Initialized Training}

    In this section, we first introduce initialized training and frozen training, two widely used training modes in continual learning. We then characterize the theoretical error and compare their performance analytically from a unified theoretical perspective. Next, we derive the optimal update conditions that minimize forgetting under a fixed model capacity. Finally, we demonstrate that frozen training outperforms initialization training when the parameter spaces are close in distance.

    \setlength{\textfloatsep}{14pt}   %
    \begin{figure}[ht]
      \centering
      \includegraphics[width=\columnwidth]{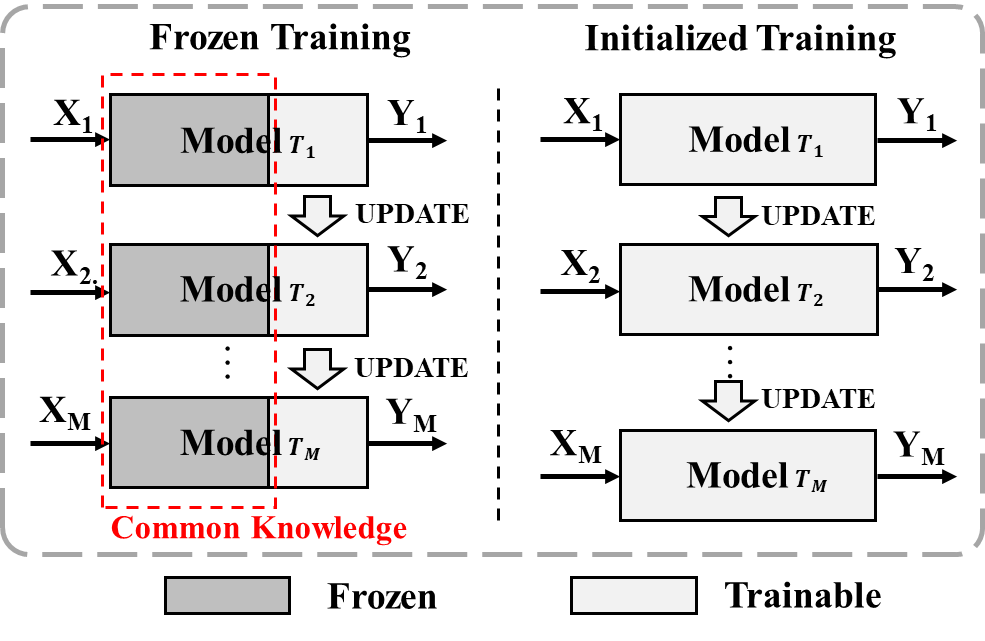}
      \caption{
        Frozen training and initialized training. Initialized training updates all parameters during training, whereas frozen training preserves general knowledge by fixing most parameters and adapts to new tasks using a few trainable parameters.
      }
      \label{icml-historical}
    \end{figure}

    Let $\tilde{u}_{(k-1)}, \tilde{q}_{(k-1)}$ denote the parameters after training the $k-1$-th task. When training the $k$-th task, a widely used approach is initialization training, in which all parameters are reoptimized starting from the previously learned model. An alternative strategy, referred to as frozen training, freezes a subset of parameters and allocates only a small number of parameters for learning new tasks. The update processes for these two training modes are formalized as follows.

   \textbf{Training Rule A ( Frozen Training ).} The knowledge acquisition is achieved by freezing the parameters $\tilde{u}_{(k-1)}$ and setting them as $\tilde{u}_{(k)}$, while training only $\tilde{q}_{(k)}$, a small set of parameters. The update process is as follows.
    \begin{equation}
    \label{正文公式_冻结训练表达式}
    \begin{aligned}
    &\arg\min_{q}||\widetilde{q}_{(k)}-\widetilde{q}_{(k-1)}||^{2},  \\
    \text{s.t.} \; Z_{(k)}^{T}q & = y_{(k)}  - X_{(k)}^{\top}\tilde u_{(k-1)} - Z_{(k)}^{\top}\tilde q_{(k-1)}
    \end{aligned}
    \end{equation}
    where $q=\widetilde{q}_{(k)}-\widetilde{q}_{(k-1)}$, and $\| u_{(k)}-u_{(j)}\|=\delta,k\neq j$ as the  parameter discrepancy over the shared subspace. Empirical analysis \cite{fewparameter1_2023} indicates that in continual learning, only a small subset of parameters is task-specific and sensitive to task changes, while most capture shared knowledge across tasks. Under frozen training, the model retains the parameters $\tilde{u}{(k-1)}$ to enable knowledge retention, and adapts to new tasks by learning the $\tilde{q}{(k)}$.

    \begin{remark}
      The \cref{正文公式_冻结训练表达式} and \cref{正文公式_初始化训练表达式} presents the expression for the case where $p_{(k)} > n_{(k)}$. The complementary case $n_{(k)} > p_{(k)}$, along with the corresponding error analysis, is deferred to Appendix A.1.
    \end{remark}

    \textbf{Training Rule B ( Initialized Training ) .} The learned parameters $(\tilde{u}_{(k-1)},\tilde{q}_{(k-1)})$ serve as initialization for current task, and the $\tilde{u}_{(k)}$ and $\tilde{q}_{(k)}$ are then joined optimized to iteratively update the learner. The update process is as follows.
    \begin{equation}
    \label{正文公式_初始化训练表达式}
    \begin{aligned}
    \operatorname*{arg\,min}_{ \tilde u_{(k)},\,\tilde q_{(k)} }
    &\ \|\tilde u_{(k)}-\tilde u_{(k-1)}\|^2
     + \|\tilde q_{(k)} - \tilde q_{(k-1)}\|^2, \\
    \text{s.t. }\;
    & X_{(k)}^\top \tilde u_{(k)} + Z_{(k)}^\top \tilde q_{(k)} = y_{(k)} .
    \end{aligned}
    \end{equation}
    
    where $\tilde u_{(k)}$ and $\tilde q_{(k)}$ are jointly optimized to achieve full fine-tuning of parameters. The following lemma provides the average difference between parameters trained and their prior true parameters under frozen training.

    \begin{lemma}
      \label{4.2引理_演化公式}
        Under frozen training, define $r = 1 - \frac{n_{(1)}}{p+p_{(1)}}$, $r_{A} = 1 - \frac{n_{(2)}}{p_{(2)}}$. After completing $(\tilde u_{(2)},\tilde q_{(2)})$, the average deviation from the previous task's parameters is given by:
        \begin{equation}
        \begin{aligned}
        &\mathbb{E}_{\tilde u,\tilde q}\!\left[
        \|\tilde u_{(2)}-u_{(1)}\|^2 + \|\tilde q_{(2)}-q_{(1)}\|^2 \right] =  \\
        &r[
        \lVert u_{(1)}\rVert^2
        + r_A \lVert q_{(1)}\rVert^2
        + \frac{n_{(2)}\lVert u_{(2)}\rVert^2}{p_{(2)}-n_{(2)}-1} ]  \\ 
        &+ (1-r_A)\lVert q_{(2)}-q_{(1)}\rVert^2 
        + \frac{n_{(2)}\sigma_{(2)}^2}{p_{(2)}-n_{(2)}-1}   \\     
        & + \frac{n_{(1)}v_{(1)}\sigma_{(1)}^2}{p+p_{(1)}-n_{(1)}-1} + \frac{n_{(2)}(1-r)\delta^2}{p_{(2)}-n_{(2)}-1} 
        \end{aligned}
        \end{equation}    
        where $v_{(1)}= 1 + \frac{p\,n_{(2)}}{(p+p_{(1)})(p_{(2)}-n_{(2)}-1)}
        - \frac{p_{(1)}n_{(2)}}{(p+p_{(1)})p_{(2)}}$.
    \end{lemma}

  Appendix A.1 provides the proof of Lemma~\ref{4.2引理_演化公式}, discusses the case where $p_{(2)} < n_{(2)}$, and presents the counterpart results for initialized training. Lemma~\ref{4.2引理_演化公式} shows that the average magnitude relative to the original parameters depends on task discrepancy $\lVert q_{(2)}-q_{(1)}\rVert^2$, the noise levels $\sigma_{(1)}$ and $\sigma_{(2)}$, and the parameter-sample ratio $r$ and $r_{A}$. The following lemma presents the average forgetting under the frozen training.

    \begin{lemma}
    \label{4.3引理_A遗忘误差}
        Under frozen training, the expected forgetting $\mathbb{E}_{\tilde u,\tilde q}[Forgetting_{FRO}]$ for any $p_{(2)}> n_{(2)}$ with respect to the previous task after parameter updates is given by: 
        \begin{equation}
        \label{正文公式_冻结训练遗忘误差}
        \begin{aligned}
         &\mathbb{E}_{\tilde u,\tilde q}[Forgetting_{FRO}] =
         \frac{n_{(2)}(1-r)\delta^2}{p_{(2)} - n_{(2)} - 1}  \\
        &+r \Bigl[(r_A - 1)\|q_{(1)}\|^2 + \frac{n_{(2)}}{p_{(2)}-n_{(2)}- 1} \|u_{(2)}\|^2 \Bigr] \\
        &+ (1-r_A)\|q_{(2)} - q_{(1)}\|^2 
        + ( \frac{p n_{(2)}}{(p+p_{(1)})(p_{(2)}-n_{(2)}-1)} \\
        &- \frac{p_{(1)} n_{(2)}}{(p+p_{(1)})p_{(2)}} ) \frac{n_{(1)}\sigma_{(1)}^2}{p+p_{(1)}-n_{(1)}-1} 
        + \frac{n_{(2)}\sigma_{(2)}^2}{p_{(2)}-n_{(2)}-1}
        \end{aligned}
        \end{equation}
    \end{lemma}

    The proof is provided in Appendix A.2. Lemma~\ref{4.3引理_A遗忘误差} shows that with $p+p_{(2)}$ fixed, increasing $p_{(2)}$ weakens general knowledge acquisition, whereas decreasing $p_{(2)}$ results in insufficient training. These two effects impose mutual constraints on the average forgetting rate, and the following theorem characterizes the minimal update condition.

    \begin{theorem}
    \label{4.4定理_EFA求导最优更新量}
    With fixed capacity $p+p_{(2)}=K$, an extremum arises with trainable $p_{(2)}$ under frozen training, provided the following condition regarding $\xi$ and $\eta$ holds.
        \begin{equation}
        \begin{aligned}
        (\xi-\eta)p_{(2)}^{2}+(n_{(2)}+1)^{2}\xi=2(n_{(2)}+1)p_{(2)}\xi
        \end{aligned}
        \end{equation} 
        And the parameter $p_{(2)}$ is given by: 
        \begin{equation}
        \begin{aligned}
        p_{(2)}=\frac{(n_{(2)}+1)(\xi+\sqrt{\xi \eta})}{\xi-\eta}
        \end{aligned}
        \end{equation}     
         Where $\xi$ and $\eta$ satisfy $\xi > \eta > 0$, $K-n_{(1)}-1>0$, and their expressions are given by the following formulas
        \begin{equation}
        \begin{aligned}
        &\xi=r\|q_{(1)}\|^{2}-\|q_{(2)}-q_{(1)}\|^{2} \\
        &\eta=r||u_{(2)}||^{2} +(1-r)\delta^{2} +\sigma_{(2)}^{2} \\
        &\qquad +\frac{n_{(1)}p\sigma_{(1)}^{2}}{(K-n_{(1)}-1)K} 
        \end{aligned}
        \end{equation}
    \end{theorem}

        \textit{Proof.} According to $\mathbb{E}_{\tilde u,\tilde q}[Forgetting_{FRO}]$ in \cref{正文公式_冻结训练遗忘误差}, taking the derivative with respect to $p_{(2)}$ yielding
        \begin{equation}
        \begin{aligned}
        &\frac{\partial [ \mathbb{E}_{\tilde u,\tilde q}]}{\partial p_{(2)}} =\frac{n_{(2)}r\|q_{(1)}\|^{2}}{p_{(2)}^{2}}
        -\frac{n_{(2)}r\|u_{(2)}\|^{2}}{(p_{(2)}-n_{(2)}-1)^{2}}  \\
        &-\frac{n_{(2)}(1-r)\delta^{2}}{(p_{(2)}-n_{(2)}-1)^{2}}
        -\frac{n_{(2)}\|q_{(2)}-q_{(1)}\|^{2}}{p_{(2)}^{2}}   \\
        &-\frac{pn_{(2)}}{K(p_{(2)}-n_{(2)}-1)^{2}}\cdot\frac{n_{(1)}\sigma_{(1)}^{2}}{(K-n_{(1)}-1)} \\
        &-\frac{n_{(2)}\sigma_{(2)}^{2}}{(p_{(2)}-n_{(2)}-1)^{2}}              
        \end{aligned}
        \end{equation}   
    Let $p_{(2)} = x$. Then, this derivative becomes a function of $p_{(2)}$. The points where this expression equals zero can be calculated to obtain the corresponding equation  
        \begin{equation}
        \begin{aligned}
         &\left[r\|q_{(1)}\|^{2}-\|q_{(2)}-q_{(1)}\|^{2}\right]\cdot(x-n_{(2)}-1)^{2}  \\
        =&x^{2}\cdot[r\|u_{(2)}\|^{2}+(1-r)\delta^{2} 
          +v_{(2)}\sigma_{(1)}^{2}+\sigma_{(2)}^{2}]   
        \end{aligned}
        \end{equation}
        where $v_{(2)}=\frac{p}{p+p_{(1)}}\cdot\frac{n_{(1)}}{p+p_{(1)}-n_{(1)}-1} $. Solving the above equation yields the solution satisfy the condition as follows
        \begin{equation}
        \begin{aligned}
        x=\frac{(n_{(2)}+1)(\xi+\sqrt{\xi \eta})}{\xi-\eta}
        \end{aligned}
        \end{equation}  
    The complete version are provided in Appendix A.3. Theorem~\ref{4.4定理_EFA求导最优更新量} shows that the average forgetting with respect to the trainable parameter $p_{(2)}$ reaches an extremum when $\xi > \eta > 0$, beyond which further increases in $p_{(2)}$ yield diminishing returns. In particular, a plateau occurs when $\|q_{(2)}-q_{(1)}\|$ is small. The following lemma describes average forgetting under initialized training.

    \begin{lemma}
      \label{4.5引理_B遗忘误差}
        Under initialized training, define $r_{B} = 1-\frac{n_{(2)}}{p+p_{(2)}}$. The expected forgetting $\mathbb{E}_{\tilde u,\tilde q}[Forgetting_{TUN}]$ relative to the previous task is expressed as follows.
        \begin{equation}
        \label{正文公式_初始化训练遗忘误差}
        \begin{aligned}
        &\mathbb{E}_{\tilde u,\tilde q}[Forgetting_{TUN}]
        = r [(r_B - 1)(\|u_{(1)}\|^2 + \|q_{(1)}\|^2)]  \\
        &   + (1-r_B)\|q_{(2)}-q_{(1)}\|^2 + (1-r_B)\delta^2    \\
        &  + \frac{n_{(1)}(r_B-1)\sigma_{(1)}^2}{p+p_{(1)}-n_{(1)}-1} + \frac{n_{(2)}\sigma_{(2)}^2}{p+p_{(2)}-n_{(2)}-1}
        \end{aligned}
        \end{equation}  
    \end{lemma}

   The proof is provided in Appendix A.4. From the above formula, $\mathbb{E}_{\tilde u,\tilde q}[Forgetting_{TUN}]$ depends on  $\|q_{(1)}\|$, $\|u_{(1)}\|$, $\sigma_{(1)}$, $\sigma_{(2)}$, $\delta$ and $\|q_{(2)}-q_{(1)}\|$, whereas $\mathbb{E}_{\tilde u,\tilde q}[Forgetting_{FRO}]$ depends on $\|q_{(1)}\|$, $\|u_{(2)}\|$, $\sigma_{(1)}$, $\sigma_{(2)}$, $\delta$ and $\|q_{(2)}-q_{(1)}\|$. Compared to \cref{正文公式_冻结训练遗忘误差}, initialized training introducing trainable parameter reduces the term involving $\sigma_{(2)}$ but also introduces term $r(r_B - 1)\|u_{(1)}\|^2$. The following theorem compares forgetting under frozen training and initialized training.

    \begin{theorem}
    \label{4.6定理_FA<FB}
     For frozen training and initialized training, the expected forgetting for any $p_{(2)}>n_{(2)}$ satisfies $\mathbb{E}_{\tilde u,\tilde q}[Forgetting_{FRO}] < \mathbb{E}_{\tilde u,\tilde q}[Forgetting_{TUN}]$ under
        \begin{equation}
        \label{正文公式_FA小于FB的条件}
        \begin{aligned}
        &\gamma_{1}\|u_{(1)}\|^{2}
        +\gamma_{2}\|u_{(2)}\|^{2}
        +r\|q_{(2)}-q_{(1)}\|^{2}  \\
        & +\gamma_{3}\delta^{2}
        +\gamma_{4}\sigma_{(1)}^{2}
        +\gamma_{5}\sigma_{(2)}^{2}<\|q_{(1)}\|^{2}
        \end{aligned}
        \end{equation}
    Where the formulaic expressions for $\gamma_{1}, \gamma_{2}, \gamma_{3}, \gamma_{4},$ and $\gamma_{5}$, together with the corresponding expressions for the constants $v_{(3)}$ through $v_{(6)}$, are provided as follows
        \begin{align}
        &\gamma_{1}=\frac{1-r_{B}}{r_{B}-r_{A}}; \\
        & \gamma_{2}=\frac{n_{(2)}}{(r_{B}-r_{A})v_{(4)}};          \\
        & \gamma_{3}=\frac{(1-r) n_{(2)}-(1-r_{B})v_{(4)}}{r(r_{B}-r_{A})v_{(4)}} ; \\
        & \gamma_{4}= \frac{ n_{(1)} \left[p n_{(2)}-(1-r_{A})p_{(1)}v_{(4)}\right] }{r(r_{B}-r_{A})(p+p_{(1)})v_{(4)}v_{(5)}}    \\
        &\qquad + \frac{ \left[(1-r_{B})(p+p_{(1)})v_{(4)} \right] n_{(1)}  }{r(r_{B}-r_{A})(p+p_{(1)})v_{(4)}v_{(5)}} ;  \notag \\
        & \gamma_{5}= \frac{p n_{(2)}}{r(r_{B}-r_{A})(p+p_{(2)}-n_{(2)}-1)v_{(4)}} ;  \\
        & v_{(3)}=1+\frac{2n_{(2)}}{p_{(2)}-n_{(2)}-1};   \\
        & v_{(4)}=p_{(2)}-n_{(2)}-1 ;    \\
        & v_{(5)}=p+p_{(1)}-n_{(1)}-1 ;  \\
        & v_{(6)}=p_{(2)}+n_{(2)}-1 .   
        \end{align}
    \end{theorem}

        \textit{Proof.} 
        Comparing the average forgetting under under the two training modes, and examining the differences across the various terms, the following inequality is obtained
        \begin{equation}
        \begin{aligned}
        &r(r_{B}-1)<0 ; \quad r_{B}>r_{A} ; \quad \frac{n_{(2)}r}{v_{(4)}}>0 ;    \\
        & \frac{n_{(2)}(1-r)}{v_{(4)}}<(1-r_{B}) ;  \quad \frac{n_{(2)}\sigma_{(2)}^{2}}{p+v_{(4)}} < \frac{n_{(2)}\sigma_{(2)}^{2}}{v_{(4)}}.            \\
        \end{aligned}
        \end{equation}
    Then, it follows that $\mathbb{E}_{\tilde u,\tilde q}[Forgetting_{FRO}]$ is smaller than $\mathbb{E}_{\tilde u,\tilde q}[Forgetting_{TUN}]$ when the following inequality holds
        \begin{equation}
        \begin{aligned}
        &\frac{(1-r_{B})\|u_{(1)}\|^{2}}{r_{B}-r_{A}}
         + \frac{n_{(2)}\|u_{(2)}\|^{2}}{(r_{B}-r_{A})v_{(4)}} + r \|q_{(2)}-q_{(1)}\|^{2} \\
        &+ \frac{ [(1-r) n_{(2)}-(1-r_{B})v_{(4)}] \delta^{2}}{r(r_{B}-r_{A})v_{(4)}} \\
        &+ \frac{ \left[p n_{(2)}-p_{(1)}(1-r_{A})v_{(4)}\right] n_{(1)} \sigma_{(1)}^{2}}{r(r_{B}-r_{A})(p+p_{(1)})v_{(4)}v_{(5)}} \\
        &+ \frac{ \left[ (1-r_{B})(p+p_{(1)})v_{(4)} \right]  n_{(1)} \sigma_{(1)}^{2}}{r(r_{B}-r_{A})(p+p_{(1)})v_{(4)}v_{(5)}} \\
        &+ \frac{p n_{(2)}\sigma_{(2)}^{2}}{r(r_{B}-r_{A})(p+v_{(4)})v_{(4)}} < ||q_{(1)}||^{2}
        \end{aligned}
        \end{equation}  

    Theorem~\ref{4.6定理_FA<FB} demonstrates that when \cref{正文公式_FA小于FB的条件} holds, the frozen training outperforms initialized training in reducing forgetting, which requires compact parameter space. The following proposition further characterizes average generalization under frozen training beyond the forgetting.

    \begin{proposition}
    \label{4.7命题_A泛化误差}
    Under frozen training, for any $p_{(2)}>n_{(2)}$, the average generalization $\mathbb{E}_{\tilde u,\tilde q}[Generalization_{FRO}]$ across learned tasks can be expressed as follows below.
    \begin{align}
    &\mathbb{E}_{\tilde u,\tilde q}[Generalization_{FRO}]
    = \frac{1-r r_A}{2} \lVert q_{(2)}-q_{(1)}\rVert^2   \\
    &+ \frac{r}{2}\Bigl[ \lVert u_{(1)} \rVert^2 + v_{(3)}
      \lVert u_{(2)}\rVert^2 + r_A \lVert q_{(1)} \rVert^2 + r_A \lVert q_{(2)} \rVert^2 \Bigr]  \notag \\
    &+\frac{(1-r)v_{(3)} \delta^2 }{2} + \frac{n_{(1)}v_{(1)}\sigma_{(1)}^2}{p+p_{(1)}-n_{(1)}-1}  + \frac{n_{(2)}\sigma_{(2)}^2}{p_{(2)}-n_{(2)}-1}  \notag
    \end{align}
    \end{proposition}

    The detailed proof is provided in Appendix A.6. The Proposition~\ref{4.7命题_A泛化误差} shows that the coefficients of $\lVert u_{(1)}\rVert^2$, $\lVert u_{(2)}\rVert^2$, $\lVert q_{(1)} \rVert^2$ and $\lVert q_{(2)} \rVert^2$ are distinct, indicating that parameter components play varying roles under frozen training. The following corollary compares frozen training and initialized training under average generalization.

    \begin{remark}
    The average generalization under initialized training, $\mathbb{E}_{\tilde u,\tilde q}[Generalization_{TUN}]$, along with the complete results for $p_{(2)}< n_{(2)}$, are detailed in Appendix A.6.
    \end{remark}

    \begin{corollary}
    \label{4.9推论_GA<GB}
       For frozen training and initialized training, the expected generalization results satisfies satisfies the inequality $\mathbb{E}_{\tilde u,\tilde q}[Generalization_{FRO}]<\mathbb{E}_{\tilde u,\tilde q}[Generalization_{TUN}]$ under the following conditions
        \begin{equation}
        \begin{aligned}
        &\frac{\beta_{1}\cdot||u_{(1)}||^{2}+\beta_{2}\cdot||u_{(2)}||^{2} +||q_{(2)}-q_{(1)}||^{2}  }{||q_{(1)}||^{2}+||q_{(2)}||^{2}} \\
        & + \frac{\beta_{3}\cdot\delta^{2}+\beta_{4}\cdot \sigma_{(1)}^{2}+\beta_{5}\cdot \sigma_{(2)}^{2}}{||q_{(1)}||^{2}+||q_{(2)}||^{2}}<1
        \end{aligned}
        \label{推论4.9不等式}
        \end{equation}
    where $\beta_{1},\beta_{2},\beta_{3},\beta_{4},\beta_{5}$ are defined as follows:
        \begin{align}
        & \beta_{1}=\frac{1-r_{B}}{r_{B}-r_{A}};            \\
        &  \beta_{2}=\frac{p_{(2)}+n_{(2)}-1-v_{(4)}r_{B}}{\left(r_{B}-r_{A}\right)v_{(4)}};           \\
        &   \beta_{3}=\frac{(1-r)v_{(6)}-(1-r r_{B})v_{(6)}}{r(r_{B}-r_{A})v_{(4)}};          \\
        &   \beta_{4}= \frac{2 \left[p n_{(2)}-(1-r_{A})p_{(1)}v_{(4)} \right]n_{(1)}}{r(r_{B}-r_{A})(p+p_{(1)})v_{(4)}v_{(5)}}  \\
        & \qquad + \frac{2 \left[(1-r_{B})(p+p_{(1)})v_{(4)} \right]n_{(1)}}{r(r_{B}-r_{A})(p+p_{(1)})v_{(4)}v_{(5)}}   ;      \notag    \\
        & \beta_{5}=\frac{2p n_{(2)}}{r(r_{B}-r_{A})(p+p_{(2)}-n_{(2)}-1)v_{(4)}} .   
        \end{align}
    \end{corollary}

        \textit{Proof.}     
        Comparing the expected generalization across the various coefficients yields the following inequality.
        \begin{align}
        & r\cdot r_{B}<r; \quad r_{B}<(1+\frac{2 n_{(2)}}{p_{(2)}-n_{(2)}-1})  ; \\
        &(1-r)(1+\frac{2n_{(2)}}{p_{(2)}-n_{(2)}-1})>(1-r\cdot r_{B}) ;    \\
        & (1+\frac{p n_{(2)}}{(p+p_{(1)})v_{(4)}}-\frac{p_{(1)}n_{(2)}}{(p+p_{(1)})p_{(2)}}) > r_{B}.  
        \end{align}
        Then, it follows that $\mathbb{E}_{\tilde u,\tilde q}[Generalization_{FRO}]$ is smaller than $\mathbb{E}_{\tilde u,\tilde q}[Generalization_{TUN}]$ provided that
        \begin{equation}
        \begin{aligned}
        &\frac{(1-r_{B})\left\|u_{(1)}\right\|^{2}}{(r_{B}-r_{A})}
        +\frac{ (v_{(6)}-v_{(4)}r_{B})\left\|u_{(2)}\right\|^{2} }{\left(r_{B}-r_{A}\right)v_{(4)}}  \\
        & +||q_{(2)}-q_{(1)}||^{2} + \frac{  r(r_{B}-1) ) v_{(6)}\delta^{2} }{r(r_{B}-r_{A})v_{(4)}}  \\
        & + \frac{2 \left[p n_{(2)}-(1-r_{A})p_{(1)}v_{(4)}\right] n_{(1)}\sigma_{(1)}^2}{r(r_{B}-r_{A})(p+p_{(1)})v_{(4)}v_{(5)}}   \\
        & + \frac{2\left[(1-r_{B})(p+p_{(1)})v_{(4)} \right]n_{(1)}\sigma_{(1)}^2}{r(r_{B}-r_{A})(p+p_{(1)})v_{(4)}v_{(5)}}   \\
        & + \frac{2p n_{(2)} \sigma_{(2)}^{2}}{r(r_{B}-r_{A})(p+p_{(2)}-n_{(2)}-1)v_{(4)}} \\
        & < \|q_{(1)}\|^{2}+\|q_{(2)}\|^{2}
        \end{aligned}
        \end{equation}

    Corollary~\ref{4.9推论_GA<GB} shows that when~\cref{推论4.9不等式} holds, frozen training achieves better average generalization than initialized training, which implies that parameter spaces are close. Therefore, although initialized training introduces additional parameters to accommodate tasks, learning is dominated by general knowledge, making frozen training more effective.

    \begin{table*}[htbp]
        \centering
        \normalsize     %%%%%%%
        \caption{Average accuracy (Avg.Acc) and forgetting rate (Forgetting) on four datasets and their variants. ‘Improvement’ indicates the gain of adaptive training over initialized training. The results are averaged over three runs (± standard deviation).}
        \resizebox{\textwidth}{!}{%
            \begin{tabular}{lcccccc}
                \specialrule{1.5pt}{0.3em}{3pt}
                \textbf{Dataset} 
                & \multicolumn{2}{c}{\textbf{Split CIFAR-100}} 
                & \multicolumn{2}{c}{\textbf{Split CUB-200}} 
                & \multicolumn{2}{c}{\textbf{Split CIFAR-10}} \\
                \cmidrule(lr){2-3} \cmidrule(lr){4-5} \cmidrule(lr){6-7}
                & Avg.Acc (\%) & Forgetting (\%) 
                & Avg.Acc (\%) & Forgetting (\%) 
                & Avg.Acc (\%) & Forgetting (\%) \\
                \midrule
                \textit{Initialized Training} 
                & 42.27 $\pm$ 0.09 & 24.84 $\pm$ 0.31  
                & 37.98 $\pm$ 0.14 & 21.28 $\pm$ 0.50 
                & 75.74 $\pm$ 0.32 & 23.30 $\pm$  0.2 \\
               \textit{Frozen Training} 
                & 43.38 $\pm$  0.11 & 23.96 $\pm$ 0.38
                & 37.15 $\pm$ 0.12 & 21.18 $\pm$ 0.20 
                & 75.98 $\pm$  0.06 & 21.10 $\pm$ 0.72 \\
                \textit{Adaptive Training}
                & \textbf{47.19 $\pm$ 0.79 } & \textbf{24.15 $\pm$ 0.27} 
                & \textbf{39.01 $\pm$ 0.42 } & \textbf{20.76 $\pm$ 0.13 } 
                & \textbf{80.55 $\pm$ 0.58 } & \textbf{19.51 $\pm$ 0.39 } \\
                \midrule
                Improvement
                & +4.92  & -0.69   
                & +1.03   & -0.52  
                & +4.81   & -3.79   \\
                \specialrule{1pt}{0.1em}{5pt}
                \specialrule{1pt}{0.1em}{3pt}
                \textbf{Dataset} 
                & \multicolumn{2}{c}{\textbf{Correlated Split CIFAR-100}} 
                & \multicolumn{2}{c}{\textbf{Corrupted Split CUB-200}} 
                & \multicolumn{2}{c}{\textbf{Split Permuted MNIST}} \\
                \midrule
                \textit{Initialized Training} 
                & 43.78 $\pm$ 0.16 & 22.45 $\pm$ 0.28 
                & 37.17 $\pm$ 0.55 & 23.21 $\pm$ 0.10 
                & 75.97 $\pm$ 0.43 & 30.76 $\pm$ 0.65 \\
               \textit{Frozen Training} 
                & 44.23 $\pm$ 0.23 & 20.84 $\pm$ 0.37 
                & 37.24 $\pm$ 0.21 & 23.96 $\pm$ 0.13 
                & 76.78 $\pm$ 0.05 & 29.21 $\pm$ 0.03 \\
                \textit{Adaptive Training}
                & \textbf{50.21 $\pm$ 0.17 } & \textbf{20.03 $\pm$ 0.12 } 
                & \textbf{ 41.26 $\pm$ 0.29 } & \textbf{ 22.76 $\pm$ 0.12 } 
                & \textbf{77.99 $\pm$ 0.61} & \textbf{28.05 $\pm$ 0.15 } \\
                \midrule
                Improvement
                & +6.43  & -2.42   
                & +4.09   & -0.45  
                & +2.02   & -2.71   \\
                \specialrule{1.5pt}{0.3em}{0pt}
            \end{tabular}%
        }
        \label{表格_六个数据集}
    \end{table*}

%%%%%%%%%%%%%%%%%%%%%%%%%%%%%%%%%%%%%%%%%%%%%%%%%%%%%%%%%
\section{Implications on Practical CL}

    In this section, motivated by the theoretical analysis presented above, we propose a hybrid adaptive update framework. The framework adaptively adjusts the training strategy based on the gradient direction from the previous task. We then validate its effectiveness through experiments with DNNs on commonly used real-world datasets. The datasets and evaluation metrics are introduced first, followed by a description of the network architecture and training details. Finally, we present and analyze the experimental results.

\subsection{Hybrid Algorithm Framework}

    As highlighted in the theoretical analysis, when tasks are closely located in the parameter space, frozen training gains more from continual learning than initialized training. This observation motivates a potential approach to improve performance: combining both strategies to enable adaptive model updates. Inspired by this, we propose a novel hybrid update framework that dynamically switches between frozen training and initialized training modes based on gradient directions in the parameter space when learning new tasks. The detailed procedure is presented in Algorithm 1.

    %%%%%%
    \begin{algorithm}[H]
      \caption{Adaptive Model Update Framework}
      \label{算法伪代码}
      \begin{algorithmic}
        \STATE \textbf{Input:} $D_1, D_2,\ldots, D_M$, $K$, $h$, $\tau$
        \STATE \textbf{Output:} $\theta_1,\theta_2,\dots,\theta_M$
        \FOR{$t = 1$ \textbf{to} $M$}
          \IF{$t = 1$}
            \STATE $\boldsymbol{\theta}_1 \leftarrow 
            \arg\min_{\boldsymbol{\theta}} \mathcal{L}(D_1, \boldsymbol{\theta})$
            \STATE Sample $x_1 \sim \mathcal{D}_1$
            \STATE $g_1 \leftarrow \nabla_{\theta} \mathcal{L}(x_1; \theta_1)$
          \ELSE
            \STATE Load $\theta_{t-1}$, Sample $x_t \sim \mathcal{D}_t$
            \STATE $g_t \leftarrow \nabla_{\theta} \mathcal{L}(x_t; \theta_{t-1})$
            \STATE $s_t \leftarrow \text{Metric Function}(g_t,\; g_{t-1})$
            \IF{$s_t > \tau$}
              \STATE $h \leftarrow \min(h + b, K)$
            \ELSE
              \STATE $h \leftarrow \max(h - b, 0)$
            \ENDIF
            \STATE $\boldsymbol{\theta}_t \leftarrow 
            \arg\min_{\boldsymbol{\theta}} \mathcal{L}(D_t, \boldsymbol{\theta})$
          \ENDIF
        \ENDFOR
      \end{algorithmic}
    \end{algorithm}

    In~\cref{算法伪代码}, when task is introduced, the relationship between the current task and previous tasks is first quantified by examining the consistency of gradient directions with respect to prior tasks. Specifically, the metric $s_t$ is computed, where $g_t$ denotes the parameter gradient under the training loss, and $Metric Function(g_t,\; g_{t-1})$ represents the metric function, which defaults to cosine similarity. Tasks with scores below threshold are considered unrelated and are assigned more modules to enhance task-specific learning.

\subsection{Datasets and Evaluation}

    We evaluate the performance of the proposed adaptive algorithm framework on four standard continual learning benchmarks and compared it with commonly used initialized training and frozen training. Specifically, (1) Permuted MNIST consists of 10 classes of 28×28 grayscale handwritten digit images, with 60,000 training samples and 10,000 test samples; (2) CIFAR-10 contains 10 classes of 32×32 color images, comprising 50,000 training samples and 10,000 test samples; (3) CIFAR-100 includes 100 classes of natural 32×32 color images, also with 50,000 training samples and 10,000 test samples; and the (4) CUB-200 is a fine-grained bird classification dataset with 200 distinct categories, consisting of 5,994 training samples and 5,974 test samples.

    The model is based on the ResNet architecture, with the network structure provided in the appendix. Specifically, the network consists of nine blocks: the first block has a single convolutional layer, while each of the remaining blocks contains two convolutional layers with residual connections. The network concludes with an average pooling layer followed by a fully connected layer. All reported results are averaged over three runs with different random seeds.

    After training on the final task, we evaluate the model by measuring the average accuracy, average current task accuracy, and forgetting rate. Let $a_{m,k}$ denote the test accuracy on task $k$ after training on task $m$. The average accuracy is defined as $Avg.Acc=\frac{1}{M}\sum_{k=1}^Ma_{M,k}$, the average current task accuracy is defined as $Task.Acc=\frac{1}{M}\sum_{k=1}^Ma_{k,k}$, and the forgetting rate is defined as $Avg.Forgetting=\frac{1}{M-1}\sum_{k=1}^{M-1} ( a_{k,k} - a_{M,k} )$. A higher average accuracy reflects stronger overall performance, a higher average current task accuracy indicates more effective learning of upcoming tasks, and a lower forgetting rate shows that the model is less susceptible to forgetting previously learned knowledge.

\subsection{Experimental Analysis}

    As shown in \cref{表格_六个数据集}, the Adaptive Training consistently outperforms Initialized Training across all four datasets in terms of both average accuracy and forgetting rate. On the original datasets, the most pronounced improvement is observed on Split CIFAR-10, where Adaptive Training yields a 4.81\% increase in average accuracy and a 3.79\% reduction in forgetting rate. We further report the average accuracy achieved when learning each task in \cref{表格_当前准确率}. The results show that Adaptive Training consistently improves current-task accuracy across all four datasets, with the largest gain of 6.58\% on Split CIFAR-100. These findings indicate that Adaptive Training not only enhances learning of incoming tasks but also effectively mitigates forgetting.

    Furthermore, according to our theoretical analysis, the advantage of Frozen Training over Initialized Training is related to the parameter-space distance between tasks. To validate this, following recent work \cite{theoryG_forgettingandgeneralization,shujuji}, we construct Correlated Split CIFAR-100 and Corrupted Split CUB-200. Specifically, we amplify task similarity in Split CIFAR-100 by introducing shared categories across tasks, and enhance task dissimilarity in Split CUB-200 through image corruption. The detailed data construction procedure is provided in Appendix B.

    As shown in the second row of \cref{表格_六个数据集}, Frozen Training demonstrates a stronger ability to resist forgetting than Initialized Training on Correlated Split CIFAR-100. On Corrupted Split CUB-200, however, the advantage of Frozen Training gradually diminishes as task dissimilarity increases. Adaptive Training more effectively integrates the strengths of both approaches, achieving a 6.43\% and 4.09\% improvement in average accuracy, along with a 2.42\% and 0.45\% reduction in forgetting rate on Correlated Split CIFAR-100 and Corrupted Split CUB-200, respectively. In summary, these results highlight the substantial potential of Adaptive Training for enhancing performance in continual learning.

    %%%
    % \vspace{-2em}
    \begin{table}[H]
        \centering
        \normalsize     %%%%%%%
        \caption{Average current task accuracy (Avg.Cur.Acc) evaluated across four datasets over three random seeds.}
        \begin{subtable}[t]{\linewidth}
            \centering
            \begin{tabular}{l@{\hskip 0.1em}c@{\hskip 0.5em}c}
                \specialrule{1.5pt}{0pt}{0.3em}
                \textbf{Dataset} & Split CIFAR-10 & Split PMNIST \\
                \midrule
                \textit{Initialized Training} 
                    & $91.83  \,{\scriptstyle \pm 0.31 }$ 
                    & $97.82  \,{\scriptstyle \pm 0.01 }$\\
                \textit{Frozen Training} 
                    & $90.44  \,{\scriptstyle \pm 0.17 }$ 
                    & $97.53  \,{\scriptstyle \pm 0.04 }$ \\
                \textit{Adaptive Training} 
                    & $94.09  \,{\scriptstyle \pm 0.10 }$ 
                    & $98.06  \,{\scriptstyle \pm 0.03 }$ \\
                \midrule
                \textbf{Improvement}   
                    & $+2.26$ 
                    & $+0.24$ \\
                \specialrule{1.5pt}{0.3em}{2.5pt}
            \end{tabular}
            \caption{Avg.Cur.Acc on Split CIFAR-10 and Split PMNIST}
            \label{tab:accuracy_subtable1}
        \end{subtable}
        \vspace{0.5em} %
        \begin{subtable}[t]{\linewidth}
            \centering
            \begin{tabular}{l@{\hskip 0.1em}c@{\hskip 0.5em}c}
                \specialrule{1.5pt}{0pt}{0.3em}
                \textbf{Dataset} & Split CIFAR-100 & Split CUB-200 \\
                \midrule
                \textit{Initialized Training} 
                    & $60.72\,{\scriptstyle \pm 0.05 }$ 
                    & $57.32\,{\scriptstyle \pm 0.36 }$\\
                \textit{Frozen Training} 
                    & $57.37\,{\scriptstyle \pm 0.13 }$ 
                    & $56.27\,{\scriptstyle \pm 0.22 }$ \\
                \textit{Adaptive Training} 
                    & $67.3\,{\scriptstyle \pm 0.74 }$ 
                    & $60.69\,{\scriptstyle \pm 0.46 }$ \\
                \midrule
                \textbf{Improvement}   
                    & $+6.58$ 
                    & $+3.37$ \\
                \specialrule{1.5pt}{0.3em}{2.5pt}
            \end{tabular}
             \caption{Avg.Cur.Acc on Split CIFAR-100 and Split CUB-200}
            \label{tab:accuracy_subtable2}
        \end{subtable}
        \label{表格_当前准确率}
    \end{table}

    To further demonstrate the effectiveness of Adaptive Training, we evaluate more challenging settings with long task sequences on Split CIFAR-100 and Split CUB-200, each consisting of 20 tasks. In \cref{表格_20个任务}, Adaptive Training mitigates forgetting more effectively than Initialized Training, while achieving better performance on individual tasks than Frozen Training, thereby combining the strengths of both approaches.

    \begin{table}[H]
        \centering
        \normalsize     %%%%%%%
        \caption{Results on Split CIFAR-100 and Split CUB-200, with each dataset consisting of a sequence of 20 training tasks.}
        % ---------- Subtable 1: Split CIFAR-100 ----------
        \begin{subtable}[t]{\linewidth}
            \centering
            \begin{tabular}{l c c c}
                \specialrule{1.5pt}{0pt}{0.3em}
                \textbf{Method} 
                & \textbf{Task.Acc} 
                & \textbf{Avg.Acc} 
                & \textbf{Forgetting} \\
                \midrule
                \textit{Initialized Training} 
                & $79.36$ 
                & $37.79$ 
                & $46.68$ \\
                \textit{Frozen Training}      
                & $78.28$ 
                & $37.59$ 
                & $45.88$ \\
                \textit{Adaptive Training}    
                & $80.81$ 
                & $39.09$ 
                & $45.08$ \\
                \midrule
                \textbf{Improvement}          
                & $+1.45$ 
                & $+1.3$ 
                & $-1.6$ \\
                \specialrule{1.5pt}{0.3em}{2.5pt}
            \end{tabular}
            \caption{Split CIFAR-100 ( 20 Tasks )}
        \end{subtable}
        \vspace{0.8em}
        % ---------- Subtable 2: CUB-200 ----------
        \begin{subtable}[t]{\linewidth}
            \centering
            \begin{tabular}{l c c c}
                \specialrule{1.5pt}{0pt}{0.3em}
                \textbf{Method} 
                & \textbf{Task.Acc} 
                & \textbf{Avg.Acc} 
                & \textbf{Forgetting} \\
                \midrule
                \textit{Initialized Training} 
                & $71.55$ 
                & $34.37$ 
                & $40.45$ \\
                \textit{Frozen Training}      
                & $69.21$ 
                & $34.90$ 
                & $36.53$ \\
                \textit{Adaptive Training}    
                & $72.59$ 
                & $36.24$ 
                & $39.60$ \\
                \midrule
                \textbf{Improvement}          
                & $+1.04$ 
                & $+1.87$ 
                & $-0.85$ \\
                \specialrule{1.5pt}{0.3em}{2.5pt}
            \end{tabular}
             \caption{Split CUB-200 ( 20 Tasks ) }
        \end{subtable}
        \label{表格_20个任务}
    \end{table}

%%%%%%%%%%%%%%%%%%%%%%%%%%%%%%%%%%%%%%%%%%%%%%%%%%%%%%%%%
\section{Conclusion}

    In this work, we systematically investigate the impact of parameter update magnitude on continual learning from a theoretical perspective, and derive optimal parameter update conditions that minimize forgetting. Through a comparative analysis of frozen training and initialized training strategies, we show that when sequential tasks are close in parameter space, frozen training leads to reduced forgetting and improved generalization. Motivated by these findings, we propose a novel hybrid parameter update framework for practical continual learning, which adaptively adjusts the update magnitude according to the gradient direction in parameter space. Experiments on deep neural networks demonstrate that the proposed approach consistently outperforms standard training. This work establishes a theoretical foundation for understanding training behaviors from the perspective of parameter update, and offers valuable guidance for designing efficient and parameter-friendly continual learning algorithms.

% %%%%%%%%%%%%%%%%%%%%%%%%%%%%%%%%%%%%%%%%%%%%%%%%%%%%%%%%%
% \clearpage  
\bibliographystyle{IEEEtran}
\bibliography{cas-refs}

\begin{thebibliography}{10}
\providecommand{\url}[1]{#1}
\csname url@samestyle\endcsname
\providecommand{\newblock}{\relax}
\providecommand{\bibinfo}[2]{#2}
\providecommand{\BIBentrySTDinterwordspacing}{\spaceskip=0pt\relax}
\providecommand{\BIBentryALTinterwordstretchfactor}{4}
\providecommand{\BIBentryALTinterwordspacing}{\spaceskip=\fontdimen2\font plus
\BIBentryALTinterwordstretchfactor\fontdimen3\font minus \fontdimen4\font\relax}
\providecommand{\BIBforeignlanguage}[2]{{%
\expandafter\ifx\csname l@#1\endcsname\relax
\typeout{** WARNING: IEEEtran.bst: No hyphenation pattern has been}%
\typeout{** loaded for the language `#1'. Using the pattern for}%
\typeout{** the default language instead.}%
\else
\language=\csname l@#1\endcsname
\fi
#2}}
\providecommand{\BIBdecl}{\relax}
\BIBdecl

\bibitem{Machine1_2015}
M.~I. Jordan and T.~Mitchell, ``Machine learning: Trends, perspectives, and prospects,'' \emph{Science}, vol. 349, pp. 255 -- 260, 2015.

\bibitem{Machine2_2021}
I.~H. Sarker, ``Machine learning: Algorithms, real-world applications and research directions,'' \emph{Sn Computer Science}, vol.~2, 2021.

\bibitem{Machine3_2023}
E.~Verwimp, S.~Ben-David, M.~Bethge, A.~Cossu, A.~Gepperth, T.~L. Hayes, E.~Hullermeier, C.~Kanan, D.~Kudithipudi, C.~H. Lampert, M.~Mundt, R.~Pascanu, A.~Popescu, A.~S. Tolias, J.~van~de Weijer, B.~Liu, V.~Lomonaco, T.~Tuytelaars, and G.~M. van~de Ven, ``Continual learning: Applications and the road forward,'' \emph{Trans. Mach. Learn. Res.}, vol. 2024, 2023.

\bibitem{OpenWorld1_2022}
Z.~Zhou, ``Open-environment machine learning,'' \emph{National Science Review}, vol.~9, 2022.

\bibitem{OpenWorld2_2023}
G.~Kim, C.~Xiao, T.~Konishi, Z.~Ke, and B.~Liu, ``Open-world continual learning: Unifying novelty detection and continual learning,'' \emph{ArXiv}, vol. abs/2304.10038, 2023.

\bibitem{OpenWorld3_2024}
F.~Zhu, S.~Ma, Z.~Cheng, X.-Y. Zhang, Z.~kui Zhang, and C.-L. Liu, ``Open-world machine learning: A review and new outlooks,'' \emph{ArXiv}, vol. abs/2403.01759, 2024.

\bibitem{incremental1_2020}
R.~Hadsell, D.~Rao, A.~A. Rusu, and R.~Pascanu, ``Embracing change: Continual learning in deep neural networks,'' \emph{Trends in Cognitive Sciences}, vol.~24, pp. 1028--1040, 2020.

\bibitem{incremental2_2022}
G.~M. van~de Ven, T.~Tuytelaars, and A.~S. Tolias, ``Three types of incremental learning,'' \emph{Nature Machine Intelligence}, vol.~4, pp. 1185 -- 1197, 2022.

\bibitem{incremental3_2023}
Y.~Wei, J.~Ye, Z.~Huang, J.~Zhang, and H.~Shan, ``Online prototype learning for online continual learning,'' \emph{2023 IEEE/CVF International Conference on Computer Vision (ICCV)}, pp. 18\,718--18\,728, 2023.

\bibitem{incremental4_2023}
Y.~Ghunaim, A.~Bibi, K.~Alhamoud, M.~Alfarra, H.~Hammoud, A.~Prabhu, P.~H.~S. Torr, and B.~Ghanem, ``Real-time evaluation in online continual learning: A new hope,'' \emph{2023 IEEE/CVF Conference on Computer Vision and Pattern Recognition (CVPR)}, pp. 11\,888--11\,897, 2023.

\bibitem{continual1_2016}
Z.~Chen and B.~Liu, ``Lifelong machine learning,'' \emph{Synthesis Lectures on Artificial Intelligence and Machine Learning}, 2016.

\bibitem{continual2_2024}
S.~Dohare, J.~F. Hernandez-Garcia, Q.~Lan, P.~Rahman, A.~R. Mahmood, and R.~S. Sutton, ``Loss of plasticity in deep continual learning,'' \emph{Nature}, vol. 632, pp. 768 -- 774, 2024.

\bibitem{continual3_2024_suvey}
B.~Wickramasinghe, G.~Saha, and K.~Roy, ``Continual learning: A review of techniques, challenges, and future directions,'' \emph{IEEE Transactions on Artificial Intelligence}, vol.~5, pp. 2526--2546, 2024.

\bibitem{human1_2021}
Q.~H. Pham, C.~Liu, and S.~C.~H. Hoi, ``Continual learning, fast and slow,'' \emph{IEEE Transactions on Pattern Analysis and Machine Intelligence}, vol.~46, pp. 134--149, 2021.

\bibitem{human2_2022}
D.~Kudithipudi, M.~Aguilar-Simon, J.~Babb, M.~Bazhenov, D.~Blackiston, J.~C. Bongard, A.~P. Brna, S.~C. Raja, N.~Cheney, J.~Clune, A.~R. Daram, S.~Fusi, P.~Helfer, L.~M. Kay, N.~A. Ketz, Z.~Kira, S.~Kolouri, J.~L. Krichmar, S.~Kriegman, M.~Levin, S.~Madireddy, S.~Manicka, A.~Marjaninejad, B.~L. McNaughton, R.~Miikkulainen, Z.~Navratilova, T.~Pandit, A.~Parker, P.~K. Pilly, S.~Risi, T.~J. Sejnowski, A.~Soltoggio, N.~Soures, A.~S. Tolias, D.~Urbina-Mel{\'e}ndez, F.~J. Valero-Cuevas, G.~M. van~de Ven, J.~T. Vogelstein, F.~Wang, R.~Weiss, A.~Yanguas-Gil, X.~Zou, and H.~T. Siegelmann, ``Biological underpinnings for lifelong learning machines,'' \emph{Nature Machine Intelligence}, vol.~4, pp. 196 -- 210, 2022.

\bibitem{human3_2023}
L.~Wang, X.~Zhang, Q.~Li, M.~Zhang, H.~Su, J.~Zhu, and Y.~Zhong, ``Incorporating neuro-inspired adaptability for continual learning in artificial intelligence,'' \emph{Nature Machine Intelligence}, vol.~5, pp. 1356 -- 1368, 2023.

\bibitem{Catastrophic1_1989}
M.~McCloskey and N.~J. Cohen, ``Catastrophic interference in connectionist networks: The sequential learning problem,'' \emph{Psychology of Learning and Motivation}, vol.~24, pp. 109--165, 1989.

\bibitem{Catastrophic2_1995}
J.~L. McClelland, B.~L. McNaughton, and R.~C. O’Reilly, ``Why there are complementary learning systems in the hippocampus and neocortex: insights from the successes and failures of connectionist models of learning and memory.'' \emph{Psychological review}, vol. 102 3, pp. 419--457, 1995.

\bibitem{Catastrophic3_2017}
S.-W. Lee, J.-H. Kim, J.~Jun, J.-W. Ha, and B.-T. Zhang, ``Overcoming catastrophic forgetting by incremental moment matching,'' \emph{ArXiv}, vol. abs/1703.08475, 2017.

\bibitem{survey1_2018}
G.~I. Parisi, R.~Kemker, J.~L. Part, C.~Kanan, and S.~Wermter, ``Continual lifelong learning with neural networks: A review,'' \emph{Neural networks : the official journal of the International Neural Network Society}, vol. 113, pp. 54--71, 2018.

\bibitem{survey3_2020}
M.~Masana, X.~Liu, B.~Twardowski, M.~Menta, A.~D. Bagdanov, and J.~van~de Weijer, ``Class-incremental learning: Survey and performance evaluation on image classification,'' \emph{IEEE Transactions on Pattern Analysis and Machine Intelligence}, vol.~45, pp. 5513--5533, 2020.

\bibitem{survey4_2023}
L.~Wang, X.~Zhang, H.~Su, and J.~Zhu, ``A comprehensive survey of continual learning: Theory, method and application,'' \emph{IEEE Transactions on Pattern Analysis and Machine Intelligence}, vol.~46, pp. 5362--5383, 2023.

\bibitem{Regularization1_2016}
J.~Kirkpatrick, R.~Pascanu, N.~C. Rabinowitz, J.~Veness, G.~Desjardins, A.~A. Rusu, K.~Milan, J.~Quan, T.~Ramalho, A.~Grabska-Barwinska, D.~Hassabis, C.~Clopath, D.~Kumaran, and R.~Hadsell, ``Overcoming catastrophic forgetting in neural networks,'' \emph{Proceedings of the National Academy of Sciences}, vol. 114, pp. 3521 -- 3526, 2016.

\bibitem{Regularization5_2018}
X.~Liu, M.~Masana, L.~Herranz, J.~van~de Weijer, A.~M. L{\'o}pez, and A.~D. Bagdanov, ``Rotate your networks: Better weight consolidation and less catastrophic forgetting,'' \emph{2018 24th International Conference on Pattern Recognition (ICPR)}, pp. 2262--2268, 2018.

\bibitem{Regularization9_2024}
W.~Chung, L.~Cherif, D.~Meger, and D.~Precup, ``Parseval regularization for continual reinforcement learning,'' \emph{ArXiv}, vol. abs/2412.07224, 2024.

\bibitem{Architecture1_2016}
A.~A. Rusu, N.~C. Rabinowitz, G.~Desjardins, H.~Soyer, J.~Kirkpatrick, K.~Kavukcuoglu, R.~Pascanu, and R.~Hadsell, ``Progressive neural networks,'' \emph{ArXiv}, vol. abs/1606.04671, 2016.

\bibitem{Architecture2_2018}
J.~Serr{\`a}, D.~Sur{\'i}s, M.~Miron, and A.~Karatzoglou, ``Overcoming catastrophic forgetting with hard attention to the task,'' in \emph{International Conference on Machine Learning}, 2018.

\bibitem{Architecture5_2023}
F.~Ye and A.~Bors, ``Self-evolved dynamic expansion model for task-free continual learning,'' \emph{2023 IEEE/CVF International Conference on Computer Vision (ICCV)}, pp. 22\,045--22\,055, 2023.

\bibitem{replay1_2016}
S.-A. Rebuffi, A.~Kolesnikov, G.~Sperl, and C.~H. Lampert, ``icarl: Incremental classifier and representation learning,'' \emph{2017 IEEE Conference on Computer Vision and Pattern Recognition (CVPR)}, pp. 5533--5542, 2016.

\bibitem{replay4_2018}
F.~M. Castro, M.~J. Mar{\'i}n-Jim{\'e}nez, N.~G. Mata, C.~Schmid, and A.~Karteek, ``End-to-end incremental learning,'' \emph{ArXiv}, vol. abs/1807.09536, 2018.

\bibitem{replay11_2024}
J.~S. Smith, L.~Valkov, S.~Halbe, V.~Gutta, R.~Feris, Z.~Kira, and L.~Karlinsky, ``Adaptive memory replay for continual learning,'' \emph{2024 IEEE/CVF Conference on Computer Vision and Pattern Recognition Workshops (CVPRW)}, pp. 3605--3615, 2024.

\bibitem{HowEfficient1_2023}
M.~Y. Harun, J.~Gallardo, T.~L. Hayes, and C.~Kanan, ``How efficient are today’s continual learning algorithms?'' \emph{2023 IEEE/CVF Conference on Computer Vision and Pattern Recognition Workshops (CVPRW)}, pp. 2431--2436, 2023.

\bibitem{Efficiency1_2023}
A.~Prabhu, H.~Hammoud, P.~K. Dokania, P.~H.~S. Torr, S.~N. Lim, B.~Ghanem, and A.~Bibi, ``Computationally budgeted continual learning: What does matter?'' \emph{2023 IEEE/CVF Conference on Computer Vision and Pattern Recognition (CVPR)}, pp. 3698--3707, 2023.

\bibitem{Efficiency2_2024}
J.~Yu, Y.~Zhuge, L.~Zhang, P.~Hu, D.~Wang, H.~Lu, and Y.~He, ``Boosting continual learning of vision-language models via mixture-of-experts adapters,'' \emph{2024 IEEE/CVF Conference on Computer Vision and Pattern Recognition (CVPR)}, pp. 23\,219--23\,230, 2024.

\bibitem{survey2_2019}
M.~D. Lange, R.~Aljundi, M.~Masana, S.~Parisot, X.~Jia, A.~Leonardis, G.~G. Slabaugh, and T.~Tuytelaars, ``A continual learning survey: Defying forgetting in classification tasks,'' \emph{IEEE Transactions on Pattern Analysis and Machine Intelligence}, vol.~44, pp. 3366--3385, 2019.

\bibitem{continual4_2024_survey}
T.~Wu, L.~Luo, Y.-F. Li, S.~Pan, T.-T. Vu, and G.~Haffari, ``Continual learning for large language models: A survey,'' \emph{ArXiv}, vol. abs/2402.01364, 2024.

\bibitem{survey5_2023}
D.-W. Zhou, Q.~Wang, Z.~Qi, H.-J. Ye, D.~chuan Zhan, and Z.~Liu, ``Class-incremental learning: A survey,'' \emph{IEEE Transactions on Pattern Analysis and Machine Intelligence}, vol.~46, pp. 9851--9873, 2023.

\bibitem{Regularization2_2017}
F.~Zenke, B.~Poole, and S.~Ganguli, ``Continual learning through synaptic intelligence,'' \emph{Proceedings of machine learning research}, vol.~70, pp. 3987--3995, 2017.

\bibitem{Regularization4_2018}
A.~Chaudhry, P.~K. Dokania, T.~Ajanthan, and P.~H.~S. Torr, ``Riemannian walk for incremental learning: Understanding forgetting and intransigence,'' \emph{ArXiv}, vol. abs/1801.10112, 2018.

\bibitem{Regularization6_2019}
H.~Ahn, D.~Lee, S.~Cha, and T.~Moon, ``Uncertainty-based continual learning with adaptive regularization,'' \emph{ArXiv}, vol. abs/1905.11614, 2019.

\bibitem{Regularization8_2023}
T.~G.~J. Rudner, F.~Bickford-Smith, Q.~Feng, Y.~W. Teh, and Y.~Gal, ``Continual learning via sequential function-space variational inference,'' \emph{ArXiv}, vol. abs/2312.17210, 2023.

\bibitem{replay2_2017}
D.~Lopez-Paz and M.~Ranzato, ``Gradient episodic memory for continual learning,'' in \emph{Neural Information Processing Systems}, 2017.

\bibitem{replay9_2020}
P.~Buzzega, M.~Boschini, A.~Porrello, D.~Abati, and S.~Calderara, ``Dark experience for general continual learning: a strong, simple baseline,'' \emph{ArXiv}, vol. abs/2004.07211, 2020.

\bibitem{replay12_2024}
J.~Liang, J.~Zhong, H.~Gu, Z.~Lu, X.~Tang, G.~Dai, S.~Huang, L.~Fan, and Q.~Yang, ``Diffusion-driven data replay: A novel approach to combat forgetting in federated class continual learning,'' in \emph{European Conference on Computer Vision}, 2024.

\bibitem{replay14_2025}
X.~Wang, S.-Y. Li, J.~Zhang, and S.~Chen, ``Cut out and replay: A simple yet versatile strategy for multi-label online continual learning,'' \emph{ArXiv}, vol. abs/2505.19680, 2025.

\bibitem{Architecture3_2020}
M.~Wortsman, V.~Ramanujan, R.~Liu, A.~Kembhavi, M.~Rastegari, J.~Yosinski, and A.~Farhadi, ``Supermasks in superposition,'' \emph{ArXiv}, vol. abs/2006.14769, 2020.

\bibitem{Architecture4_2022}
R.~Ardywibowo, Z.~Huo, Z.~Wang, B.~J. Mortazavi, S.~Huang, and X.~Qian, ``Varigrow: Variational architecture growing for task-agnostic continual learning based on bayesian novelty,'' in \emph{International Conference on Machine Learning}, 2022.

\bibitem{Architecture6_2024}
J.~Thapa and R.~Li, ``Bayesian adaptation of network depth and width for continual learning,'' in \emph{International Conference on Machine Learning}, 2024.

\bibitem{Architecture9_2025}
F.~Ye, A.~Bors, and K.~Zhang, ``Dynamic expansion diffusion learning for lifelong generative modelling,'' in \emph{AAAI Conference on Artificial Intelligence}, 2025.

\bibitem{pretain1_2021}
A.~Douillard, A.~Ram'e, G.~Couairon, and M.~Cord, ``Dytox: Transformers for continual learning with dynamic token expansion,'' \emph{2022 IEEE/CVF Conference on Computer Vision and Pattern Recognition (CVPR)}, pp. 9275--9285, 2021.

\bibitem{pretain2_2023}
Q.~Gao, C.~Zhao, Y.~Sun, T.~Xi, G.~Zhang, B.~Ghanem, and J.~Zhang, ``A unified continual learning framework with general parameter-efficient tuning,'' \emph{2023 IEEE/CVF International Conference on Computer Vision (ICCV)}, pp. 11\,449--11\,459, 2023.

\bibitem{pretain4_2024}
D.~Marczak, B.~Twardowski, T.~Trzci'nski, and S.~Cygert, ``Magmax: Leveraging model merging for seamless continual learning,'' \emph{ArXiv}, vol. abs/2407.06322, 2024.

\bibitem{theory1_2017}
C.~V. Nguyen, Y.~Li, T.~D. Bui, and R.~E. Turner, ``Variational continual learning,'' \emph{ArXiv}, vol. abs/1710.10628, 2017.

\bibitem{theory2_2022}
D.~B. Prado and P.~Riddle, ``A theory for knowledge transfer in continual learning,'' \emph{ArXiv}, vol. abs/2208.06931, 2022.

\bibitem{theory3_2023}
T.~G.~J. Rudner, F.~Bickford-Smith, Q.~Feng, Y.~W. Teh, and Y.~Gal, ``Continual learning via sequential function-space variational inference,'' \emph{ArXiv}, vol. abs/2312.17210, 2023.

\bibitem{theoryA_functionalregularisation}
P.~Pan, S.~Swaroop, A.~Immer, R.~Eschenhagen, R.~Turner, and M.~E. Khan, ``Continual deep learning by functional regularisation of memorable past,'' \emph{Advances in neural information processing systems}, vol.~33, pp. 4453--4464, 2020.

\bibitem{theoryB_teacherstudent}
S.~Lee, S.~Goldt, and A.~Saxe, ``Continual learning in the teacher-student setup: Impact of task similarity,'' in \emph{International Conference on Machine Learning}.\hskip 1em plus 0.5em minus 0.4em\relax PMLR, 2021, pp. 6109--6119.

\bibitem{theoryC_Howcatastrophic}
I.~Evron, E.~Moroshko, R.~Ward, N.~Srebro, and D.~Soudry, ``How catastrophic can catastrophic forgetting be in linear regression?'' in \emph{Conference on Learning Theory}.\hskip 1em plus 0.5em minus 0.4em\relax PMLR, 2022, pp. 4028--4079.

\bibitem{theoryD_theoreticalstudyOOD}
G.~Kim, C.~Xiao, T.~Konishi, Z.~Ke, and B.~Liu, ``A theoretical study on solving continual learning,'' \emph{Advances in neural information processing systems}, vol.~35, pp. 5065--5079, 2022.

\bibitem{theoryE_promptbased}
L.~Wang, J.~Xie, X.~Zhang, M.~Huang, H.~Su, and J.~Zhu, ``Hierarchical decomposition of prompt-based continual learning: Rethinking obscured sub-optimality,'' \emph{Advances in Neural Information Processing Systems}, vol.~36, pp. 69\,054--69\,076, 2023.

\bibitem{theoryF_Learnability}
G.~Kim, C.~Xiao, T.~Konishi, and B.~Liu, ``Learnability and algorithm for continual learning,'' in \emph{International conference on machine learning}.\hskip 1em plus 0.5em minus 0.4em\relax PMLR, 2023, pp. 16\,877--16\,896.

\bibitem{theoryG_forgettingandgeneralization}
S.~Lin, P.~Ju, Y.~Liang, and N.~Shroff, ``Theory on forgetting and generalization of continual learning,'' in \emph{International Conference on Machine Learning}.\hskip 1em plus 0.5em minus 0.4em\relax PMLR, 2023, pp. 21\,078--21\,100.

\bibitem{qianyi}
P.~Ju, S.~Lin, M.~S. Squillante, Y.~Liang, and N.~Shroff, ``Theoretical analysis on the generalization power of overfitted transfer learning.''

\bibitem{theoryH_plasticity}
S.~Dohare, J.~F. Hernandez-Garcia, Q.~Lan, P.~Rahman, A.~R. Mahmood, and R.~S. Sutton, ``Loss of plasticity in deep continual learning,'' \emph{Nature}, vol. 632, no. 8026, pp. 768--774, 2024.

\bibitem{theoryI_Measuringforgetting}
J.~Kim, Y.~Kim, and J.-y. Sohn, ``Measuring representational shifts in continual learning: A linear transformation perspective,'' in \emph{Forty-second International Conference on Machine Learning}, 2025.

\bibitem{theory4_2024}
X.~Zhao, H.~Wang, W.~Huang, and W.~Lin, ``A statistical theory of regularization-based continual learning,'' in \emph{International Conference on Machine Learning}.\hskip 1em plus 0.5em minus 0.4em\relax PMLR, 2024, pp. 61\,021--61\,039.

\bibitem{fewparameter1_2023}
H.~Zhao, T.~Zhou, G.~Long, J.~Jiang, and C.~Zhang, ``Does continual learning equally forget all parameters?'' in \emph{International Conference on Machine Learning}, 2023.

\bibitem{shujuji}
J.~Deng, Q.~Wu, P.~Ju, S.~Lin, Y.~Liang, and N.~Shroff, ``Unlocking the power of rehearsal in continual learning: A theoretical perspective,'' \emph{arXiv preprint arXiv:2506.00205}, 2025.

\end{thebibliography}

\end{document}